\documentclass{article}
\usepackage{colordvi,epsfig,amssymb}
\usepackage{amsmath}
\usepackage[table]{xcolor}   
\usepackage[linesnumbered, ruled,noend]{algorithm2e}
\usepackage{multicol}
\usepackage{mathtools}  
\usepackage{subcaption}

\newcommand{\R}{\mathbb{R}}
\newcommand{\N}{\mathbb{N}}

\newcommand{\qed}{\nobreak \ifvmode \relax \else
	\ifdim\lastskip<1.5em \hskip-\lastskip
	\hskip1.5em plus0em minus0.5em \fi \nobreak
	\vrule height0.75em width0.5em depth0.25em\fi}

		\newcommand{\eqdef}{\overset{\text{def}}{=}} 
		\newcommand{\E}[1]{\mathbf{E}\left[#1\right] } 
		\newcommand{\norm}[1]{\lVert#1\rVert}
		\newcommand{\dotprod}[1]{\left< #1\right>}

		\newtheorem{definition}{Definition}
		\newtheorem{theorem}[definition]{Theorem}
		\newtheorem{proposition}[definition]{Proposition}
		\newtheorem{lemma}[definition]{Lemma}
		\newtheorem{corollary}[definition]{Corollary}

		\usepackage[colorinlistoftodos,bordercolor=orange,backgroundcolor=orange!20,linecolor=orange,textsize=scriptsize]{todonotes}

\usepackage[maxbibnames=99,maxcitenames=3,doi=false,isbn=false,url=false,style=authoryear,backend=bibtex]{biblatex}
\bibliography{sketch_Optimal_transport}
\AtEveryBibitem{%
  \clearlist{publisher}%
  \clearname{editor}%
  \clearfield{editor}
}
\makeatletter 
\def\blx@maxline{77}
\makeatother
\title{Greedy stochastic algorithms for entropy-regularized optimal transport problems}
\author{Brahim Khalil Abid \\
\'Ecole polytechnique\footnote{BKA carried out this work while interning at Inria and CREST/ENSAE} \\ \texttt{brahim-khalil.abid@polytechnique.edu}  \and  Robert M. Gower\\ LTCI, T\'el\'ecom-Paristech, Université Paris-Saclay\footnote{RMG carried out this work while at Inria in the SIERRA team, funded by the Fondation de Sciences Math\'ematiques de Paris (FSMP) }\\ \texttt{robert.gower@telecom-paristech.fr}}

\begin{document}
\maketitle

\begin{abstract}

%

Optimal transport (OT) distances are finding evermore applications in machine learning and computer vision, but their wide spread use in larger-scale problems is impeded by their high computational cost. In this work we develop a family of fast and practical stochastic algorithms for solving the optimal transport problem with an entropic penalization. This work extends the recently developed Greenkhorn algorithm, in the sense that, the Greenkhorn algorithm is a limiting case of this family. We also provide a simple and general convergence theorem for all algorithms in the class, with rates that match the best known rates of Greenkorn and the Sinkhorn algorithm, and conclude with numerical experiments that show under what regime of penalization the new stochastic methods are faster than the aforementioned methods.
\end{abstract}

\section{Introduction}
Probability distributions are the backbone of machine learning and statistics: we use them to represent a variety of objects in learning tasks, ranging from statistical models to data representations. Comparing different distributions is often done using information divergences such as Kullback-Leiber divergence, yet this discards much of structural and geometric information present in the distribution. Developing a practical measure that captures the geometry of the probability distribution is a problem to which the optimal transport (OT) distance offers an attractive solution.

First formulated by~\cite{monge1781memoire}  then revisited by \cite{Kantorovich1942}, OT distances have the inherent particularity of capturing the geometrical properties of the probability measures. However, they have a drawback: computing an OT distance has a typical cost of the order $O(n^3 \log n)$ for histograms of $n$ points~(\cite{PeleW09}). This prevents the application of OT distances in large-scale machine learning problems.

The idea of entropy penalization, proposed by~\cite{Cuturi2013}, represents a key milestone in this field. The benefits of such a regularization scheme are multiple: the regularized problem has a unique solution, greater computational stability, and can be solved efficiently using the Sinkhorn algorithm. 
 This new family of distances has been used in a wide range of applications, such as image classification~(\cite{Cuturi2013}), unsupervised learning using Restricted Boltzmann Machines~(\cite{Montavon2016}), learning with a Wasserstein Loss (~\cite{Wasserstein15}), domain adaptation~(\cite{courty:hal-01018698}), computer graphics~(\cite{solomon:hal-01188953}), and neuroimaging ~(\cite{GramfortCuturiPC15}). The growing interest in applications for the Sinkhorn distances has sparked the development of new and efficient algorithms for its calculation, such as stochastic gradient based algorithms by~\cite{GenevayCuturiBach2016}, and fast methods to compute Wasserstein barycenters~(\cite{pmlr-v32-cuturi14}). To this end, \cite{AltschulerWR:2017} have developed the Greenkhorn algorithm, a greedy variant of the  Sinkhorn algorithm that selects columns and rows to be updated that most violate the constraints. The authors present both promising numerical results, besting the Sinkhorn algorithm, and an insightful theoretical complexity that is linear in $n$.
 

%


\textbf{Our contribution}: We expand on the idea of greedy column and row selection by proposing a family of algorithms that assign a probability of updating each row and column. Moreover, our family allows for any sampling so long as the probabilities are proportional to the violation of each column or row with respect to the transport polytope. We call our algorithm the Greedy Stochastic Sinkhorn. We explain the idea  behind this family of methods, show how Greenkhorn is a limiting case, and propose several other instances of the algorithm. We develop an all encompassing convergence theorem that recovers the best known $O(n/\epsilon^2)$ iteration complexity for the Greenkhorn algorithm. Finally, we exhibit some numerical experiments that explicit the relevance of Greedy Stochastic Sinkhorn in a particular regime of penalization, along with a discussion around the computational properties of the algorithms.

\subsection{The Optimal Transport problem}
The discrete OT problem can be seen as a problem of optimal resource allocation given by a linear program
\begin{eqnarray}
	T^* &\in & \arg \min_{T \in \R^{n \times n}_+} \dotprod{T,C},\nonumber \\
	& & \mbox{subject to}\quad  T {\bf 1} =r,\,\, T^\top {\bf 1} = c, \label{eq:nonregOT}
\end{eqnarray}
where $r,c \in \Delta_n \eqdef \{ x \in \R^n\, | \, \sum_{i=1}^n x_i= 1\}$ are respectively the initial and target distributions, $C \in \R_+^{n \times n}$ the transport cost matrix and ${\bf 1}$ is a vector of all ones of an appropriate dimension. Matrices $T\in \R^{n\times n}_+$ that satisfy the transport constraints in~\eqref{eq:nonregOT} represent valid transportation maps between $r$ and $c$, where $T_{ij}$ will represent the mass transported from $r_i$ to $c_j$. The matrix $T^*$ is a transportation map that minimizes the transportation cost, the computed minimum $\dotprod{T^*,C}$ is the optimal transport value and it defines a distance between $r$ and $c$ ~(\cite{Villani}). The transportation map $T^*$ can be computed using the  network simplex or interior point methods~(\cite{PeleW09}), but the computational cost is in both cases $O(n^3\log(n))$. It is this cubic cost in the dimension that makes this notion of distance infeasible in high-dimensional  settings, such as in computer vision or high dimensional inference.

\subsection{Entropic regularization and the Sinkhorn algorithm}
An interesting approach to alleviate the computational burden was proposed by ~\cite{Cuturi2013} through the introduction of an entropic regularization as follows
\begin{eqnarray}
	T_{\lambda}^* &=& \arg \min_{T \in \R^{n \times n}_+} \dotprod{T,C} - \frac{1}{\lambda}E(T), \nonumber \\
	& & \mbox{subject to}\quad  T {\bf 1} =r,\,\, T^\top {\bf 1} = c, \label{eq:regOT}
\end{eqnarray}
where the entropy is $E(T) = \sum_{i,j=1}^n - T_{ij} \log(T_{ij})$. 
Due to the strong convexity introduced by the entropic regularization, the problem~\eqref{eq:regOT} now has a unique solution. What is more, using duality theory~\eqref{eq:regOT} has a smooth and unconstrained dual formulation.
Leveraging on the dual~\citeauthor{Cuturi2013} showed that~\eqref{eq:regOT} can be equivalently re-written as the following matrix scaling problem:
find $u,v \in \R^n_+ $ such that
\begin{equation}\label{eq:matscale}
D(u)AD(v){\bf 1} =r  \quad \mbox{and} \quad D(v)A^\top D(u){\bf 1} =c,
\end{equation}
where $A = e^{-\lambda C}$ with the exponential taken element-wise and  $D(u)$ denotes a diagonal matrix with the elements of $u$ on the diagonal. With the $(u,v)$ solution to~\eqref{eq:matscale}, the solution to~\eqref{eq:regOT} is simply given by $T_{\lambda}^* = D(u)AD(v).$ This matrix scaling problem can now be efficiently solved using the celebrated Sinkhorn algorithm, as proposed by~\cite{Cuturi2013}.

The Sinkhorn algorithm is a fixed point iteration algorithm for solving~\eqref{eq:matscale} which alternately scales the  row and column sums to match the desired marginals
\begin{eqnarray}
	u^{k+1} &=& r./(Av^k),  \nonumber \\.
	v^{k+1} &=& c./(A^\top u^k), \label{eq:Sinkh}
\end{eqnarray}
where we have used $x./y$ to denote elementwise division of vectors\footnote{In other words $x./y= D(y)^{-1}x$}.
 On top of being a simple and fast algorithm, the Sinkhorn algorithm is also GPU-friendly since its highest cost is a matrix vector product which can be parallelized. The resulting distance $\dotprod{T_{\lambda}^*,C}$ defined by~\eqref{eq:regOT} has been dubbed the Sinkhorn distance. 
 

{\bf Notation:} For the sake of brevity, we use $r(T) = T {\bf 1}$ and $c(T) =T^\top {\bf 1}$ to denote the \emph{row sum} and \emph{column sum} vectors of $T$, respectively.
Let $U_{r,c}$ be the transport polytope defined by
\begin{equation*}
U_{r,c} \eqdef \{ T \in \R^{n \times n}_+\, |\, r(T) = r, \,\,c(T) = c \}.
\end{equation*}

Since we need to solve~\eqref{eq:matscale}, in order to discuss convergence results, we need to define a distance that measures how far are the scaled iterates from the transport polytope $U_{r,c}$. We will use in all the following work the $\ell_1$ distance
\begin{equation}\label{eq:dist}
	dist(A,U_{r,c}) \eqdef \norm{r(A) - r}_1 + \norm{c(A) - c}_1
\end{equation}
which, as argued by~\citeauthor{AltschulerWR:2017}, is much more suitable to compare probability distributions than the $\ell_2$ distance. A simple example to see this: taking $p = (\frac{1}{n},...\frac{1}{n},0....0) \in \Delta_{2n}$ and $q = (0,0....0,\frac{1}{n},...\frac{1}{n}) \in \Delta_{2n}$ as two probability distributions with disjoint supports, we see that $\norm{p-q}_1 =2$, while $\norm{p-q}_2 = \frac{1}{\sqrt{n}}$ and thus decreases as $n$ increases, despite being clearly distinct distributions for all $n$.


\subsection{Greedy Sinkhorn: Greenkhorn algorithm}

Greenkhorn is a greedy version of Sinkhorn proposed by~\cite{AltschulerWR:2017} where at each iteration only one coordinate of $u$ or $v$ is updated in~\eqref{eq:Sinkh}, picking each time the one with highest violation with respect to the corresponding marginal. These violations are computed with the following function
\begin{eqnarray}
	\rho(a,b) &= &b - a + a \log(\frac{a}{b}), \quad \mbox{for } a,b \in \R_+ \\
	d_\rho (u,v) &= &\sum_{i=1}^n \rho(u_i,v_i),
	\quad \mbox{for }u,v \in \R^n_+. \label{def:divrho}
\end{eqnarray}
For vectors in the simplex $u,v \in \Delta_n$, we have that $d_\rho(u,v)$ coincides with the Kullback-Leiber divergence. For this reason, $d_\rho$ is known as the generalized Kullback-Leiber divergence. It is not a distance because it is not symmetric, but it verifies $d_\rho \geq 0$ and $d_\rho(u,v)=0 \Leftrightarrow u=v$. Therefore, if $u,v$ are two vectors of positive entries, $d_\rho(u,v)$ will return some measurement on how far they are from each other. Let $\rho^r(M)$ (resp. $\rho^c(M)$) be the vector of the row sum violations (resp. column sum violations) of a given matrix $M \in \R^{n\times n}_+$ with respect to $r$ (resp. $c$), where the violations are computed using $\rho$, that is
\begin{eqnarray}
	\rho^r(M) = \left( \rho(r_i,r_i(M)) \right)_{i=1..n} \in \R_+^n, \nonumber \\
	\rho^c(M) = \left( \rho(c_i,c_i(M)) \right)_{i=1..n} \in \R_+^n. \nonumber
\end{eqnarray}
We will refer to the concatenation of these two vectors as the \emph{marginal violations} denoted by 
\begin{equation}
\rho(M) = \left( \rho^r(M) , \rho^c(M) \right)_{i=1..n} \in \R_+^{2n}. \label{eq:marg_viol}
\end{equation}
The marginal violations vector $\rho(M)$  measures how far the matrix $M$ is from the transport polytope $U_{r,c}$ in the sense that $M \in U_{r,c}$ if and only if all entries of $\rho(M)$ are equal to zero.

The Greenkhorn algorithm uses $\rho(M)$ to guide the choice of which 
row or column should be updated.  As the name indicates, the algorithm chooses the row or column index greedily, that is the index of maximal value in $\rho(M)$, see Algorithm~\ref{alg:greenkhorn}. This greedy variation of Sinkhorn is expected to perform better in practice, mainly because it does not update rows or columns that already match the correct marginal sum value. 


\begin{algorithm}
	\KwData{$A\in \R^{n\times n}_+$, $r,c\in \R_+^n$,  $\epsilon > 0$}
	initialization: $u$,$v$ = \textbf{1}\\
	\While{  dist($D(u)AD(v)$,$U_{r,c}) \geq \epsilon$}{
	$I = \arg \max_{i=1..2n} \rho(D(u)AD(v))$ \\
	\eIf{$I\leq n$ (corresponds a row update)}{
		$u_I = r_I./(Av)_I$ \\
	}{
		$v_{I-n} = c_{I-n}./(A^\top u)_{I-n}$ \\
	}
	\KwResult{$u,v\in \R^n_+$ such that $D(u)AD(v) \in U_{r,c} $ }
}
\caption{Greenkhorn}
\label{alg:greenkhorn}
\end{algorithm}

\citeauthor{AltschulerWR:2017} proved that, to reach an $\epsilon>0$ approximate solution, the Greenkhorn algorithm and the Sinkhorn algorithm converge in at most $28n\epsilon^{-2} \log(\frac{s}{l})$  and $28\epsilon^{-2} \log(\frac{s}{l}) $ iterations, respectively, where $s = \norm{A}_1$ is the total mass and $l$ the smallest entry of the matrix $A$. \citeauthor{AltschulerWR:2017} also claimed that the Greenkhorn algorithm can be implemented in such a way that the iteration cost is linear in $n$, consequently the overall complexity of either the Sinkhorn algorithm or Greenkhorn is quadratic in $n$, which is stark contrast to the cubic dependency of the interior point type methods~(\cite{PeleW09}). Since the authors omitted the details on how such a linear iteration complexity can be achieved,  we have given the details in Section~\ref{sec:complexiter}.
  The $\epsilon^{-2}$  dependency of Greenkhorn and Sinkhorn is also in contrast with logarithmic dependency on $\epsilon$ in interior point based methods.
   Thus Greenkhorn and Sinkhorn are well suited for the large dimensional setting where we can tolerate an approximate solution. This is typically the case in the problems that we are interested in here, such as problems that arise in large dimensional machine learning.
  
%


\section{Greedy Stochastic Sinkhorn}

While the greedy strategy in the Greenkhorn algorithm is, in some sense, optimal for one step, it may not be the best strategy over a number of iterations. Here we introduce a more flexible, and less aggressive updating strategy.

At each iteration of the \emph{Greedy Stochastic Sinkhorn} algorithm, instead of picking the column or row with the highest violation, as is done in the Greenkhorn algorithm, we will assign to each row and column a probability of being updated.
Because we want the columns and rows with highest violation to be updated more frequently, we assign a higher probability to columns and rows with a higher violation. We do this using an \emph{increasing probability function}. 
\begin{definition}
 We say that $\Psi$ is a increasing probability function if there exists an increasing positive function $g: \R_+ \mapsto \R_+$ such that
\begin{equation}
\forall h \in \R_+^{2n} \quad \Psi(h) = \left( \frac{g(h_k)}{\sum_{i=1}^{2n} g(h_i) } \right)_{k=1..2n} \in \Delta_{2n}. \label{def:monoprobafunc}
\end{equation}
\end{definition}
Several examples of an increasing probability function are given as follows
\begin{eqnarray}
\Psi(h) &=& \left( \frac{1}{2n} \right)_{i=1, \ldots, 2n}, \label{eq:uni}\\
\Psi(h) &=& \left( \frac{h_i^\alpha}{\sum_{j=1..2n}h_j^\alpha} \right)_{i=1..2n}, \label{eq:polyn} \\
\Psi(h)& =& \left( \frac{e^{(h_i/T)}}{\sum_{j=1..2n}e^{(h_j/T)}} \right)_{i=1..2n},  \label{eq:softmax}
\end{eqnarray}
where $T,\alpha > 0$ are parameters.
If $\rho$ is our current vector of violations, then $\Psi(\rho) = p \in \Delta_{2n}$ defines a probability distribution. Furthermore, since $\Psi$ is built on top of an increasing function, a larger violation $\rho_i$ will result in a larger probability $p_i$. See Algorithm~\ref{alg:StochSink} for the pseudocode of this family of stochastic algorithms. In the next section we prove that Algorithm~\ref{alg:StochSink} converges for any increasing probability function.
This is particularly interesting when we consider that the Greenkhorn algorithm is a limiting case of the Greedy Stochastic Sinkhorn.
  Indeed, the selection criteria of the Greenhkorn algorithm corresponds to taking the limit over $\alpha \rightarrow \infty $ of the probability function~\eqref{eq:polyn}.

\section{Convergence analysis}
We now present our main convergence theorem, discuss its consequences and proof.
\begin{theorem}\label{theo:genconv}
Consider the sequence of matrices $A^k \eqdef D(u^k)AD(v^k)$ produced by Algorithm~\ref{alg:StochSink} with an increasing probability function $\Psi$ as defined in~\eqref{def:monoprobafunc}.
Then for a given $\epsilon>0,$ we have that
\begin{equation}\label{eq:genconv}
\exists k \in \N, \quad k \leq \frac{28n}{\epsilon^2}\log\left(\frac{s}{\ell}\right),\end{equation}
such that $\E{dist(A^k,U_{r,c})} \leq \epsilon.$
\end{theorem}
We make several interesting remarks on the consequence of this theorem.
\begin{enumerate}
\item Since $dist(A^k,U_{r,c})$ is a positive random variable, by Markov's inequality we have that the convergence in expectation given in Theorem~\ref{theo:genconv} also proves that $dist(A^k,U_{r,c})$ converges in probability to zero. The variance also converges to zero at a $O(\left.n\right/ \epsilon)$ rate, as proven in Section~\ref{sec:var}.
\item The convergence rate given in Theorem~\ref{theo:genconv} is exactly the same rate as given by~\cite{AltschulerWR:2017} for the Greenkorn algorithm.
\item Remarkably the rate of convergence does not depend on the choice of probability function $\Psi$. Thus, in theory, a uniform selection of the coordinates gives the same asymptotic convergence as the Greenkhorn selection criteria.
\end{enumerate}


\begin{algorithm}[h]
    \KwData{ $A\in \R^{n\times n}_+$, $r,c\in \R_+^n$, $\Psi$, $\epsilon$}
    \KwResult{$u,v\in R^n_+$ such that $D(u)AD(v) \in U_{r,c}$}
    initialization: $u$,$v$ = \textbf{1}\\
    \While{ dist($D(u)AD(v)$,$U_{r,c}) \geq \epsilon$}{
     $p = \Psi(\rho(D(u)AD(v)))  \in \Delta_{2n}$\\
     Sample index $I$ with $P(I=i)=p_i, \quad \forall i\in \{1,2, \ldots, 2n \}$ \\
    \eIf{$I\leq n$ (corresponds a row update)}{
    $u_{I} = r_{I}./(Av)_{I}$ \\}{
    $v_{I-n} = c_{I-n}./(A^\top u)_{I-n}$ \\}
    }
\caption{Greedy Stochastic Sinkhorn}
\label{alg:StochSink}
\end{algorithm}

Before moving onto the proof, we need several auxiliary lemmas.
\subsection{Useful lemmas}

Our analysis is based on the dual objective of~\eqref{eq:nonregOT} given by
\begin{equation}\label{eq:dualobj}
	f(x,y) = \sum_{i,j= 1}^{n} A_{ij}e^{x_i+y_j} - \dotprod{r,x} - \dotprod{c,y}.
\end{equation}
Let 	$X=D(e^x)$ and $Y=D(e^y)$. By writing out the first order optimality conditions of $f(x,y)$ we arrive at
\begin{equation}\label{eq:matscale2} r(XAY) = r \quad \mbox{and} \quad c(XAY) =c\;.\end{equation}	 
That is, the row sum and column sum of $XAY$ is $r$ and $c$, respectively. By denoting $u=e^x$ and $v=e^y$ the given scaling vectors, we see that~\eqref{eq:matscale2} is the matrix scaling problem~\eqref{eq:matscale}.
Throughout this section we use $(u^k,v^k)$ to denote the $(u,v)$ vectors of Algorithm~\ref{alg:StochSink} after completing the $k$th iteration. We also denote $(x^k,y^k) = (\log(u^k),\log(v^k)).$

The proof of Theorem~\ref{theo:genconv} is based on the four next lemmas. 

The first lemma was presented by \cite{AltschulerWR:2017}, and it links the evolution of  dual objective~\eqref{eq:dualobj} and the marginal violations.
\begin{lemma} \label{lem:altschuler}
	For a given $k$, if $(u^{k+1},v^{k+1})$ was obtained by updating coordinate $I$ of $u^k$ then the following identity holds
    $$f(x^k,y^k) - f(x^{k+1},y^{k+1}) = \rho(r_I,r_I(D(u^k)AD(v^k))),$$
	and if they were obtained by updating coordinate J of $v^k$ then
$$f(x^k,y^k) - f(x^{k+1},y^{k+1}) = \rho(c_J,c_J(D(u^k)AD(v^k))).$$
Since $\rho \geq 0$ then the sequence of real numbers $(f(x^k,y^k))_{k \in \textbf{N}}$ is decreasing.
\end{lemma}

Next we have an extension  to the stochastic setting of another lemma by~\citeauthor{AltschulerWR:2017}. It links the expectation of the dual objective value to a type of condition number of the $A$ matrix.

\begin{lemma}\label{lem:fzerosl}
	Let $((u^k,v^k))_{k \in \textbf{N}}$ and the associated $((x^k,y^k))_{k \in \textbf{N}}$ be a sequence of scaling vectors produced by the Greedy Stochastic Sinkhorn Algorithm~\ref{alg:StochSink}. Then the following inequalities hold
	\begin{eqnarray}
		\E{f(x^k,y^k)} - \min_{x,y\in R} f(x,y) &\leq &f(0,0) - \min_{x,y\in R} f(x,y) \nonumber \\
		&\leq & \log \left(\frac{s}{l}\right), \label{eq:asd89j98qj}
	\end{eqnarray}
where $l = \min_{i,j}|A_{ij}|$ and $s = \norm{A}_1$. As a direct consequence, we also have
\begin{equation} \label{eq:a9k982a22}
f(0,0) - \E{f(x^k,y^k)} \leq\log \left(\frac{s}{l}\right)\; .
\end{equation}
\end{lemma}

To prove Lemma~\ref{lem:fzerosl} we will use this result from~\cite{AltschulerWR:2017}.
\begin{proposition}\label{prop:tempp}
Let $x^0,y^0 = 0$ the initial points and $x^1,y^1$ the points resulting from updating a single coordinate $u^k$ or a single coordinate of $v^k$. Then the following inequality holds
	\begin{equation}
		f(x^1,y^1) - \min_{x,y\in R} f(x,y) \leq f(0,0) - \min_{x,y\in R} f(x,y) \leq \log \left(\frac{s}{l}\right) \nonumber
	\end{equation}
where $l$ is the smallest entry of $A$ and $s = \norm{A}_1$.
\end{proposition}

Now we write the proof of Lemma~\ref{lem:fzerosl} of our paper.

\begin{proof}
We have from Lemma~\ref{lem:altschuler} that the sequence of real numbers $(f(x^k,y^k))_{k \in \textbf{N}}$ is decreasing for all sequences of updated indexes, in particular in the stochastic setting. Therefore the inequality of the previous proposition in fact holds for all iterations
    \begin{equation*}
     f(x^k,y^k) - \min f \leq f(0,0) - \min f \leq  \log \left(\frac{s}{l}\right).
    \end{equation*}
Taking expectation in the previous inequalities gives~\eqref{eq:asd89j98qj}. To prove~\eqref{eq:a9k982a22} we have that:
\begin{eqnarray*}
	- \E{f(x^k,y^k)} &\leq &- \min f \\
	f(0,0) - \E{f(x^k,y^k)} &\leq & f(0,0) - \min f \\
	&\overset{\eqref{eq:asd89j98qj}}{\leq} &\log \left(\frac{s}{l}\right),
\end{eqnarray*}
which concludes the proof. $\qed$
\end{proof}

The next lemma is a useful inequality on ordered series of real number.
\begin{lemma}[Chebyshev inequality]\label{lem:cheby}
Let $a_1\leq a_2 \leq \ldots \leq a_n \in \R$ and $b_1 \leq b_2 \leq \ldots \leq b_n$ be two ordered sequences. It follows that
\begin{equation}\label{eq:ineqcheb}
\frac{1}{n}\sum_{i=1}^n a_i b_i \geq \left(\frac{1}{n}\sum_{j=1}^n a_j \right) \left( \frac{1}{n}\sum_{j=1}^n b_j \right).
\end{equation}
\end{lemma}
\begin{proof}
Simply note the identity
\begin{align*}
& \sum_{i,j=1}^n (a_i-a_j)(b_i-b_j) 
\\& = \sum_{j=1}^n\left( \sum_{i=1}^n\left(a_ib_i   -a_ib_j -a_ja_i\right) +n a_jb_j \right) 
\\&= 2n\sum_{i=1}^n a_ib_i  -2\sum_{i=1}^n a_i\sum_{j=1}^n b_j. 
\end{align*}
The proof now follows by noting that $(a_i-a_j)(b_i-b_j)$ is positive for every $i,j=1,\ldots, n.$. $\qed$
\end{proof}

Finally we have a lemma that is a generalization of the Pinsker inequality. 
\begin{lemma}\label{lem:pinskergen}
The following generalized Pinsker inequality holds for $v,u \in R^n_+$
			\begin{equation}\label{eq:pinskergen}
			 \norm{u - v}_1 \leq \sqrt{7\norm{u}_1 d_\rho(u,v)}\;, 
			\end{equation}
where $d_\rho(u,v)$ is defined as in~\eqref{def:divrho}.
\end{lemma}

Lemma~\ref{lem:pinskergen} is a generalization of the next proposition from~\cite{AltschulerWR:2017}.

\begin{proposition}\label{prop:asiha7shd}
For $v\in \Delta_n$ and $u \in R^{n}_{+}$ we have that
\begin{equation}\label{eq:aiod2hd8o}
	||v -u||_1 \leq \sqrt{7 d_\rho(v,u)} 
 \end{equation}
\end{proposition}
Using Proposition~\ref{prop:asiha7shd} we will now prove Lemma~\ref{lem:pinskergen}.
 
\begin{proof}
First we note that $\rho$ is a $1-$homogeneous function, that is
\begin{eqnarray*}
	\rho(tx,ty) &= &ty - tx + tx \log{\frac{tx}{ty}} \\
	&= &t(y - x + x\log{\frac{x}{y}}) \\
	&= &t \rho(x,y) \quad \forall x,y,t > 0.
\end{eqnarray*}
This implies immediatly that $d_\rho$ is $1-$homogeneous as well. Therefore for $v,u \in R^{n}_{+}$ and we have that
\begin{eqnarray*}
	d_\rho(v,u) &= &\norm{v}_1 d _\rho(\frac{v}{\norm{v}_1},\frac{u}{\norm{v}_1}). \end{eqnarray*}
	Furthermore
\begin{eqnarray*}	
	\norm{v - u}_1 &= &\norm{v}_1\norm{\frac{v}{\norm{v}_1}-\frac{u}{\norm{v}_1}}_1 \\
	&\overset{\eqref{eq:aiod2hd8o}}{\leq} &\norm{v}_1 \sqrt{7 d_\rho(\frac{v}{\norm{v}_1},\frac{u}{\norm{v}_1})}. 
\end{eqnarray*}
Consequently
\begin{eqnarray*}	
	\norm{v - u}_1 &\leq &\sqrt{7\norm{v}_1 d_\rho(v,u)}.
\end{eqnarray*}
which concludes the proof. $\qed$
\end{proof}

\subsection{Proof of Theorem~\ref{theo:genconv}}

\begin{proof}
Let $D_k \eqdef \E{dist(A^k,U_{r,c})} $ and let $k^* \in \N$ be an integer such that $D_k > \epsilon$ for all $k < k^*$ (in other terms, an index such that the algorithm has not converged yet at the corresponding iteration).

Recall that $\rho(A^k)$ is the vector of all $2n$ marginal violations for the matrix $A^k$, as defined in~\eqref{eq:marg_viol}. We will write its components as $\rho_i(A^k)$ for a given index $i$. Recall that $\Psi(\rho(A^k))$ is the vector of probabilities of picking each row and column, and similarly we will write its components $\Psi_i(\rho(A^k))$, which is then the probability of picking index $i$.
We start the proof by showing that $D_k^2$ is upper bounded by the following conditional expectation
\begin{eqnarray*}
	\E{\rho_I(A^k) \mid A^k }  &= &\sum_{i=1}^{2n} \Psi_i(\rho(A^k))  \rho_i(A^k),
\end{eqnarray*}
where $I$ is the index randomly sampled at iteration $k$.
Let $k < k^*$ and since we assume that~\eqref{def:monoprobafunc} holds for some function $g$, we have that
\begin{align}
\E{\rho_I(A^k) \mid A^k } &= \sum_{i=1}^{2n} \frac{g(\rho_i(A^k))}{\sum_{j=1}^{2n}g(\rho_j(A^k))}  \rho_i(A^k) \nonumber \\
& \overset{\eqref{eq:ineqcheb}}{\geq }  \frac{1}{n}\sum_{i=1}^{2n} \rho_i(A^k) \nonumber \\
& \overset{\eqref{eq:pinskergen}}{\geq}  \frac{\left(\norm{r-r(A^k)}_1 + \norm{c- c(A^k)}_1\right)^2 }{28n}, \label{eq:andh23nun}
\end{align}
where we applied Lemma~\ref{lem:cheby} in the first inequality which relies on the monotonicity of $g$, and  the generalized Pinsker inequality~\eqref{eq:pinskergen} in the second inequality with $a = (r,c)$, $b = (r(A^k),c(A^k))$ and we used that $\norm{a}_1 = \norm{r}_1+\norm{c}_1 =2.$ Taking expectation in~\eqref{eq:andh23nun}, using the law of total expectation and the fact that $\E{X^2} \geq \E{X}^2$ for any random variable $X$ gives
\begin{eqnarray}
  \E{\rho_I(A^k)} & \overset{\eqref{eq:andh23nun}}{\geq} &
\frac{\E{\norm{r-r(A^k)}_1 + \norm{c- c(A^k)}_1}^2}{28n} \nonumber\\
& =& \frac{1}{28n} D_k^2 
\quad > \quad  \frac{\epsilon^2}{28n}.  \label{eq:whoah3a}
\end{eqnarray}
To conclude, we now use Lemma~\ref{lem:altschuler} to re-write $\E{\rho_I(A^k) \mid A^k }$ as
\begin{eqnarray}
\E{f(x^{k},y^{k}) - f(x^{k+1},y^{k+1}) \mid x^{k},y^{k}} \qquad \qquad && \nonumber \\
\quad = \sum_{i=1}^{2n} \Psi_i(\rho(A^k))  \rho_i(A^k) 
= \E{\rho_I(A^k) \mid A^k }.&&\label{eq:asnohas}
\end{eqnarray}
Thus taking expectation in~\eqref{eq:asnohas} gives
\begin{eqnarray}
\E{f(x^{k},y^{k}) - f(x^{k+1},y^{k+1})  } &= &\E{\rho_I(A^k)} \nonumber  \\ 
&\overset{\eqref{eq:whoah3a}}{>} &\frac{\epsilon^2}{28n}.
\label{eq:asnoha22s}
\end{eqnarray}
	 Summing over $k=0,\ldots, k^*-1$ in~\eqref{eq:asnoha22s} and using telescopic cancellation we have that
	\begin{equation}
f(x^{0},y^{0}) - \E{f(x^{k^*},y^{k^*}) } > \frac{k^* \epsilon^2}{28 n}.
\label{eq:ssdsdsiij}
\end{equation}
Combining the above with 
\[f(0,0) - \E{f(x^k,y^k)} \leq\log\left(\frac{s}{\ell}\right), \]
 as proven in Lemma~\ref{lem:fzerosl}, we have that
\begin{equation}
\frac{28 n}{\epsilon^2}log\left(\frac{s}{\ell}\right) > k^*.
\end{equation}

This proves that for a given integer $k^*$
\begin{equation}
 \forall k < k^*, \quad D_k> \epsilon, \quad
\Rightarrow \quad \frac{28 n}{\epsilon^2}log\left(\frac{s}{\ell}\right) > k^*.
\end{equation}
The contrapositive of the above statement is given by
\begin{equation}
k^* \geq \frac{28 n}{\epsilon^2}log\left(\frac{s}{\ell}\right) \Rightarrow \exists k < k^*, \quad D_k \leq \epsilon. \label{eq:last}
\end{equation}
Choosing $k^* = \lceil \frac{28 n}{\epsilon^2}log\left(\frac{s}{\ell}\right) \rceil$ 
 concludes the proof.
$\qed$
\end{proof}
   
\begin{corollary}
If we choose $\Psi$ as either~\eqref{eq:uni},~\eqref{eq:polyn} or~\eqref{eq:softmax} then the Greedy Stochastic Sinkhorn algorithm converges at a $O(\frac{n}{\epsilon^2})$ rate according to Theorem~\ref{theo:genconv}.
\end{corollary}   
   \begin{proof}
   The proof follows by observing that~\eqref{eq:uni},~\eqref{eq:polyn} or~\eqref{eq:softmax} are increasing probability functions. That is, the functions $x \mapsto e^{x/T}$, $x \mapsto x^\alpha$ and $x \mapsto 1$ are positive increasing real-valued functions.
   \end{proof}

\subsection{Variance convergence}\label{sec:var}
In this section we prove that the variance of the Greedy Stochastic Sinkhorn algorithm also converges to zero, and furthermore, we show that the rate of convergence is $O(n/\epsilon)$ to reach a variance of $\epsilon>0.$ Having already established the convergence of the Greedy Stochastic Sinkhorn algorithm in expectation and high probability, together with the following convergence of the variance, we can conclude that the convergence is qualitatively almost deterministic. 


\begin{theorem}
Consider the sequence of matrices $A^k \eqdef D(u^k)AD(v^k)$ produced by Stochastin Sinkhorn Algorithm with an increasing probability function $\Psi$.
Then for a given $\epsilon>0,$ we have that
\[\exists k \in \N, \quad k \leq \frac{28n}{\epsilon}\log\left(\frac{s}{\ell}\right),\]
such that $\mathrm{Var} (dist(A^k,U_{r,c}) ) \leq \epsilon.$
\end{theorem}

\begin{proof}
First we note that
\begin{equation*}
0 \leq \mathrm{Var} (dist(A^k,U_{r,c}) ) \leq \E{dist(A^k,U_{r,c})^2}.
\end{equation*}
Therefore we will prove that $\E{dist(A^k,U_{r,c})^2}$ converges at a certain rate, which will give the same thing for the variance.
\newline
The proof is extremely similar to the proof of Theorem 2, therefore we will only highlight where we have slight differences. Now we call $D'_k = \E{dist(A^k,U_{r,c})^2}$ and consider an integer $k^*$ such that $\forall k < k^*, D'_k > \epsilon$ where $\epsilon > 0$ is fixed.
\newline
The only modification of the proof is that now we will have
\begin{eqnarray*}
  \E{\rho_I(A^k)} &\geq &
\frac{\E{dist(A^k,U_{r,c})^2}}{28n} \nonumber\\
& =& \frac{1}{28n} D'_k \nonumber \\
&\geq  &\frac{\epsilon}{28n}. 
\end{eqnarray*}
Note that in the proof of Theorem 2 we had a $\epsilon^2$ in place of $\epsilon$ in the above lower bound. Taking this minor difference into account, the rest of the proof follows verbatim to the proof of Theorem 2. \qed

\end{proof}

\section{Numerical experiments}\label{sec:numerics}
In this section we provide some empirical insights into the behaviour of the Greedy Stochastic Sinkhorn Algorithm~\ref{alg:StochSink}. We consider both real and synthetic datasets: MNIST digits and random histograms. The authors of both~\cite{Cuturi2013} and~\cite{AltschulerWR:2017} provided numerical experiments where Sinkhorn and Greenkhorn perform considerably better than other Optimal Transport algorithms. We will show how the Greedy Stochastic Sinkhorn has a similar efficiency for monotonic probability functions~\eqref{def:monoprobafunc}. In particular, we will show that for~\eqref{eq:polyn} with $\alpha = 1$, Greedy Stochastic Sinkhorn outperforms Greenkhorn in the short-term for regimes of small penalization. Finally we will discuss the computational properties of the three algorithms, in particular some drawbacks of Greenkhorn and Greedy Stochastic Sinkhorn in comparaison with Sinkhorn, and give insights on how to bypass them. But first, we explicity show how Greenkhorn and Greedy Stochastic Sinkorn have linear iteration complexities.
\subsection{Updating the marginal violation}\label{sec:complexiter}
The Greenkhorn Algorithm~\ref{alg:greenkhorn} and the Greedy Stochastic Sinkhorn
Algorithm~\ref{alg:StochSink} must re-compute marginal violations 
\[\rho(A^k) = \left[  \rho(r_i,r_i(A^k)) _{i=1..n} 
, \, \rho(c_i,c_i(A^k))_{i=1..n}  \right],\]
at each iteration. Calculating $\rho(A^k)$ from scratch at each iteration would cost $O(n^2)$, which would defeat the purpose of both algorithms of having a linear
iteration complexity. Fortunately $\rho(A^k)$ can be updated on the fly with only $O(n)$ operations. To see this, suppose we have stored  $c(A^k), r(A^k)$ and $\rho(A^k)$ and now we wish to calculate $\rho(A^{k+1}).$
Suppose we sample an index $I$ in Algorithm~\ref{alg:StochSink} such that $I\in \{1,\ldots, n \}$, consequently we update
\begin{equation}\label{eq:uindexk}
u_{I}^{k+1} = r_{I}./(Av^k)_{I} \, ,
 \end{equation}
while $v^{k+1} =v^k$ and $u_i^{k+1} = u^k_i$ for $i \neq I$ remain unaltered. We can thus calculate the $i$th component of $r(A^{k+1})$ via
\begin{align*}
 r_i(A^{k+1}) &= (D(u^{k+1})AD(v^{k}){\bf 1})_i =
u^{k+1}_i A_{i:}v^k\\
&\overset{\eqref{eq:uindexk}}{=} 
\begin{cases}
u^{k+1}_I A_{I:}v^k & \mbox{if  } i =I,\\
r_i(A^{k}) & \mbox{if  } i \neq I.
\end{cases}
\end{align*}
The column sum vector can be updated using
\begin{align*}
c(A^{k+1}) &= D(v^k) \sum_{i=1}^n A_{i:} u_i^{k+1}\\
&= D(v^k)  A_{I:} u_I^{k+1} -D(v^k)  A_{i:} u_i^{k} + c(A^k).
\end{align*}
Thus both $r(A^{k+1})$ and $c(A^{k+1})$ can be updated using $O(n)$ operations. Since $\rho(r_i,r_i(A^{k+1})) = \rho(r_i,r_i(A^k))$ for $i \neq I$, only $n+1$ components of $\rho(A^{k+1})$ need to re-computed, which costs $O(n)$ operations.
The $O(n)$ cost of the case where $I \in \{n+1, \ldots 2n\}$ can be deduced in an   analogous way.

\subsection{Experiments}

We perform experiments on MNIST dataset. We take pairs of elements from the $28\times 28$ pixels MNIST dataset, that we then vectorize into 1D arrays $r$ and $c$ (in the sense that, for example, the 2 by 2 matrix $((1,2),(3,4))$ becomes the vector $(1,2,3,4)$ ).  The cost matrix $C$ is then constructed so that $C_{ij}$ equals the $\ell_1$ distance between pixels $i$ and $j$ in the $28\times 28$ grid. 
We then apply the Sinkhorn, Greenkhorn and Greedy Stochastic Sinkhorn algorithms to compute a diagonal scaling of $A = e^{-\lambda C}$.  This process is then repeated $20$ times, for each time we randomly sample a pair of images from the MNIST dataset. Finally, we report the average performance of the algorithms over these $20$ experiments.


    The choice of $\lambda$ defines the penalization, and we highlight the fact that regimes of low penalization, corresponding to higher values of $\lambda$~\eqref{eq:regOT}, are of particular interest since they change the least the solution of the original non-regularized problem~\eqref{eq:nonregOT}. For this setting, Greedy Stochastic Sinkhorn with ~\eqref{eq:polyn} for $\alpha = 1$ is clearly the best choice overall, see Figure~\ref{fig:1}.
\newcommand\x{0.41} 
 \begin{figure}[h]
 \centering
  \includegraphics[width=\x\textwidth]{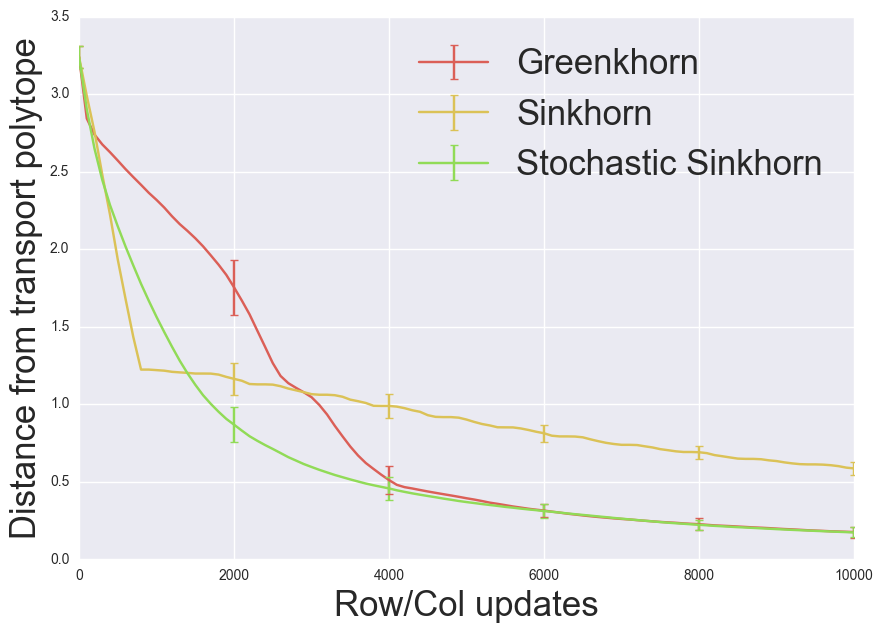}
 \caption{Evolution of distance from transport polytope for Sinkhorn, Greenkhorn and Greedy Stochastic Sinkhorn  (~\eqref{eq:polyn} with $\alpha=1$) in regimes of low penalization ($\lambda = 10$) on MNIST dataset. For the x-axis, one should read ``number of row and column updates'' in the sense that one iteration on the x-axis represents one update of a row or a column.}
 \label{fig:1}
 \end{figure}
 
We also compared in Figure~\ref{fig:limitingcase} the Greedy Stochastic Sinkhorn for various choices of parameters. As expected, using the probability function~\eqref{eq:polyn} with $\alpha \rightarrow + \infty$ or~\eqref{eq:softmax} with $T \rightarrow 0$, the Greedy Stochastic Sinkhorn algorithm reduces to the Greenkhorn algorithm

 Notice also that the standard deviation for Greedy Stochastic Sinkhorn (represented as errorbars) tends to $0$, which is a very important property because of the stochastic nature of the algorithms. In fact, this means that not only the expectation of the distance tends to zero, but also the variance, a fact which we prove in the appendix.

\begin{figure}[h]
\centering
\begin{subfigure}[h]{\x\textwidth}
   \includegraphics[width=1\textwidth]{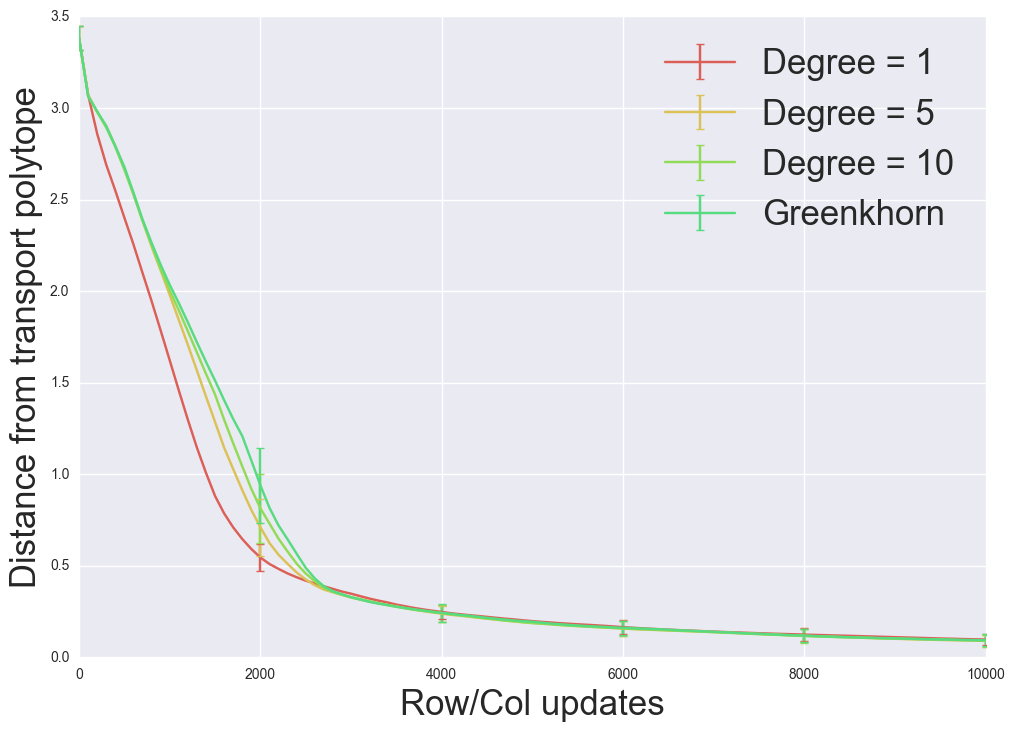}
   \label{fig:polyn} 
\end{subfigure}

\begin{subfigure}[h]{\x\textwidth}
   \includegraphics[width=1\textwidth]{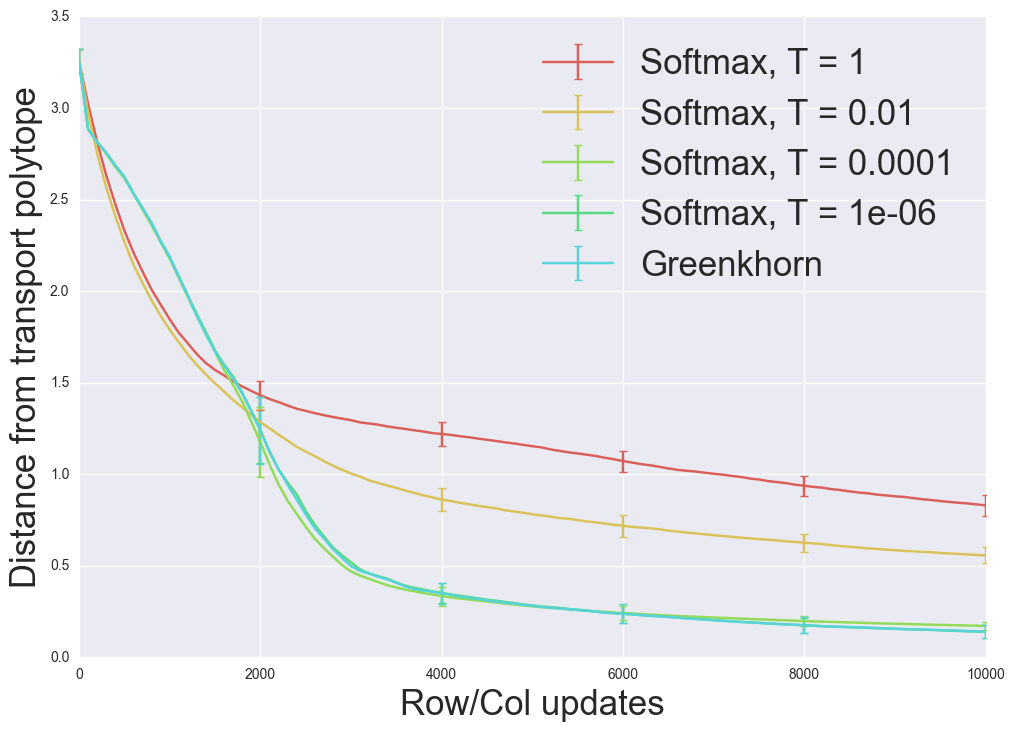}
\label{fig:softmax}
\end{subfigure}
\caption{Greedy Stochastic Sinkhorn with different probability functions, and Greenkhorn as limiting case. Up: polynomial probabilities~\eqref{eq:polyn}, down: softmax probabilities~\eqref{eq:softmax}. For the x-axis, one should read ``number of row and column updates'' in the sense that one iteration on the x-axis represents one update of a row or a column. } \label{fig:limitingcase}
\end{figure}

\subsection{Discussion and block algorithms}
While Greedy Stochastic Sinkhorn and Greenkhorn empirically perform better than Sinkhorn, they also have two computational drawbacks: firstly they are not parallelizable. In fact, each iteration of Sinkhorn~\eqref{eq:Sinkh} is a rescaling of $u$ and $v$ involving a matrix-vector product that can be parallelized. Greenkhorn and Greedy Stochastic Sinkhorn update only one element per iteration which does not involve a similar product that we can parallelize. Secondly, the greedy algorithms compute at each iteration the marginal violations, which, despite only costing $O(n)$, it does represent an additional computational cost per iteration. 
\newline

One solution to these issues is to re-compute these marginal violations only once every $d$ iterations: for example in Greenkhorn we compute marginal violations $\rho(A^k)$ and we update not only the index of highest value, but the $d$ indexes of highest values. Something similar can be done in Greedy Stochastic Sinkhorn by sampling $d$ indexes without replacement instead of just one. By doing so, on the one hand we reduce computation time for computing the marginal violations by a factor $d$, and on the other hand the algorithms described are now parallelizable because updating $d$ components of $u$ and $v$ does involve a matrix-vector product.
 
 \begin{figure}[h]
 \centering
  \includegraphics[width=\x\textwidth]{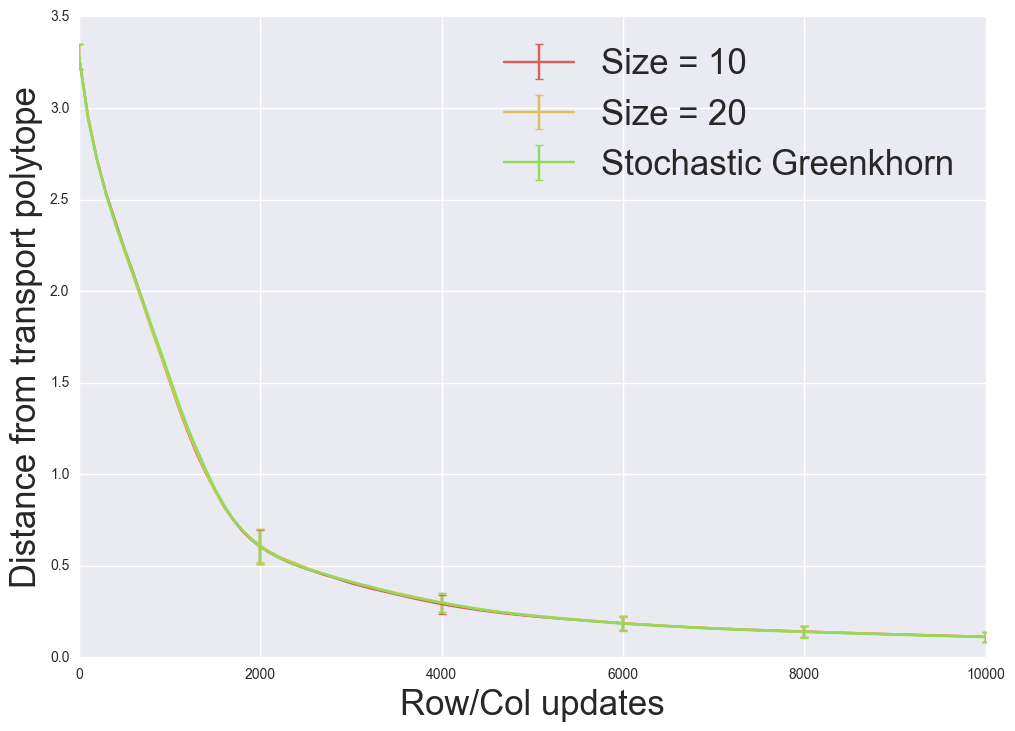}
 \caption{Evolution of distance from transport polytope for Block Greedy Stochastic Sinkhorn compared to Greedy Stochastic Sinkhorn (~\eqref{eq:polyn} with $\alpha=1$). For the x-axis, one should read ``number of row and column updates'' in the sense that one iteration on the x-axis represents one update of a row or a column.} \label{fig:block}
 \end{figure}

This idea is motivated by the numerical results of~\citeauthor{AltschulerWR:2017} where the authors concluded that the efficiency of Greenkhorn is mainly due to the fact that is does not update rows and columns that already match desired sums, more than the fact that it updates indexes with highest marginal violations. This means that the procedure of updating $d$ indexes instead of just one is expected to have a similar efficiency, which is indeed the result we get in our own numerical experiments as shown in Figure~\ref{fig:block}.

\section{Conclusion}
We presented  a family of stochastic algorithms for entropy-regularized OT problems. We were able to derive convergence rates for a very broad class of probability functions, along with numerical experiments where a simple and intuitive choice of probability functions performed the best. 
We also proposed and tested simple numerical solutions to the drawbacks of the greedy algorithms. 

\subsubsection*{Acknowledgements}
Both authors are indebted to Marco Cuturi for essentially teaching both of them about OT and many inspiring discussions.
RMG acknowledges the support of the FSMP postdoctoral fund.  BKA carried our this work while at an internship, and acknowledges CREST/ENSAE for funding, and the SIERRA team, Inria Paris, for hosting.

\printbibliography

\end{document}